\documentclass[11pt]{article}

\usepackage[final]{acl}

\usepackage{times}
\usepackage{latexsym}

\usepackage[T1]{fontenc}

\usepackage[utf8]{inputenc}

\usepackage{microtype}

\usepackage{inconsolata}

\usepackage{graphicx}

\usepackage[utf8]{inputenc}
\usepackage[T1]{fontenc}
\usepackage{url}
\usepackage{booktabs}
\usepackage{amsfonts}
\usepackage{nicefrac}
\usepackage{microtype}
\usepackage[table]{xcolor}   
\usepackage{enumitem}        
\usepackage{multirow}
\usepackage{graphicx}
\usepackage{bm}
\usepackage{amsmath}
\usepackage{mathtools}
\usepackage{amssymb}
\usepackage{wrapfig}
\usepackage{amsthm}
\usepackage{subcaption}
\usepackage{bbm}

\PassOptionsToPackage{hidelinks}{hyperref}

\usepackage{hyperref}         

\definecolor{mygray}{gray}{.92}
\newtheorem{proposition}{Proposition}
\def\eg{\emph{e.g.}}

\newcommand{\ours}{\text{MPD}}

\newcommand{\X}{\mathbf{X}}
\newcommand{\U}{\mathbf{U}}
\newcommand{\V}{\mathbf{V}}
\newcommand{\Q}{\mathbf{Q}}
\newcommand{\w}{\mathbf{w}}

\newcommand{\bP}{\mathbf{P}}

\newcommand{\tableCellHeight}{1}
\newcommand{\tabstyle}[1]{
  \setlength{\tabcolsep}{#1}
  \renewcommand{\arraystretch}{\tableCellHeight}
  \centering
  \small
}

\usepackage{algorithm}
\usepackage{algpseudocode}

\algrenewcommand\algorithmicrequire{\textbf{Input:}}
\algrenewcommand\algorithmicensure{\textbf{Output:}}

\newcommand{\IState}{\State\hspace*{-\algorithmicindent}}

\title{Mitigating Hallucinations in Large Vision-Language Models without Performance Degradation}

\author{
  \textbf{Xingyu Zhu\textsuperscript{1,2}},
  \textbf{Junfeng Fang\textsuperscript{2}\thanks{Corresponding authors}},
  \textbf{Shuo Wang\textsuperscript{1}}, \\
  \textbf{Beier Zhu\textsuperscript{1}},
  \textbf{Zhicai Wang\textsuperscript{1}},
  \textbf{Yonghui Yang\textsuperscript{2}},
  \textbf{Xiangnan He\textsuperscript{1}\footnotemark[1]}
\\
\\
  \textsuperscript{1}
  MoE Key Lab of BIPC, University of Science and Technology of China,\\
  \textsuperscript{2}
  National University of Singapore
}

\begin{document}
\maketitle

\begin{abstract}
Large Vision-Language Models (LVLMs) exhibit powerful generative capabilities but frequently produce hallucinations that compromise output reliability. Fine-tuning on annotated data devoid of hallucinations offers the most direct solution, while its high computational cost motivates recent representation-based methods, which focus on mitigating hallucinatory components within hidden representations. Though efficient, we empirically observe that these methods degrade general generation capacity due to incomplete extraction of hallucination components and non-selective parameter updates. To address these limitations, we propose MPD, a dual-stage framework for mitigating hallucinations without performance degradation. Specifically, our MPD relies on two essential factors: (1) semantic-aware component disentanglement to extract pure hallucination components, and (2) interpretable parameter updates that selectively modify parameters most relevant to hallucination. Extensive experiments demonstrate that MPD achieves state-of-the-art performance, reducing hallucinations by 23.4\% while maintaining 97.4\% of general generative capability as evaluated on LLaVA-Bench and MME, with no additional computational cost.
\end{abstract}

\section{Introduction}
Recent advances in large vision-language
models (LVLMs)~\cite{zhu2023minigpt, liu2023visual, ye2023mplug, dai2024instructblip, bai2023qwen, zhu2026principled} have demonstrated remarkable capabilities in cross-modal understanding and generation. However, these models exhibit a persistent limitation known as hallucination phenomena~\cite{gunjal2024detecting,liu2023mitigating,zhu2026look}, \textit{i.e.}, a critical divergence where generated textual descriptions systematically misrepresent visual content. Typical manifestations include fabricating non-existent objects, misattributing visual properties, or hallucinating erroneous spatial relationships within images. These systematic errors not only undermine practical applications requiring precise vision-language alignment, but also pose significant risks for misinformation propagation and safety-critical deployment scenarios~\cite{liu2024survey, chen2024detecting}.

\begin{figure*}[htbp]
        \centering
        \includegraphics[width=1\linewidth]{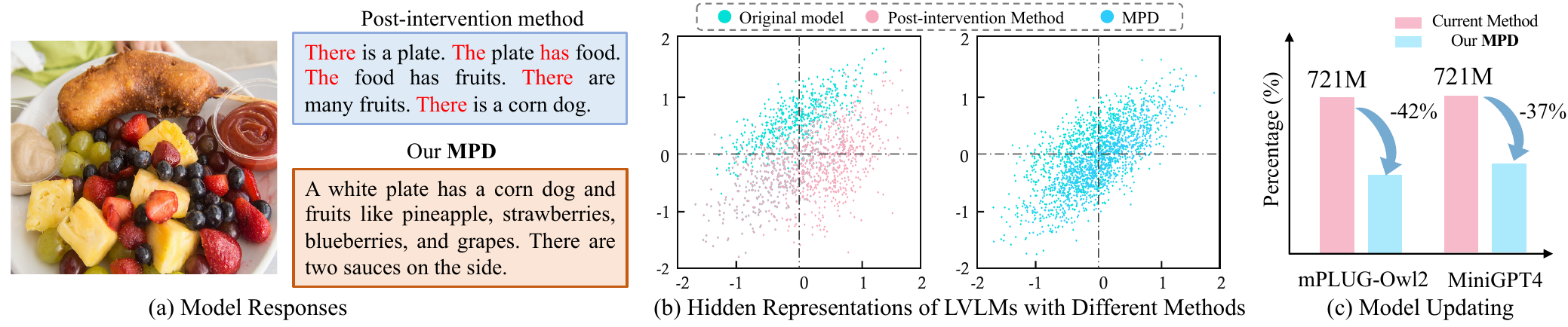}
        \caption{
        \textbf{Comparison between conventional representation-intervention methods and MPD.} 
        \textbf{(a)} Model responses to the same input image show that traditional methods produce repetitive and incoherent descriptions, while MPD generates coherent and informative outputs.
        \textbf{(b)} Post-intervention hidden representations from prior methods\cite{yang2025nullu,liu2024reducing} drift away from the original distribution, 
        \textbf{(c)} MPD aligns the hidden representations more closely with the original distribution. 
        \textbf{(d)} MPD also requires significantly fewer parameter updates and FLOPs, enabling more efficient hallucination mitigation.
        }\label{fig:motivation}
\vspace{-10pt}
\end{figure*}

Conventional approaches~\cite{chen2024halc, leng2024mitigating, chen2024alleviating} to hallucination mitigation primarily rely on labor-intensive dataset curation through manual annotation to filter hallucinatory content for fine-tuning, inherently suffering from time-consuming bottlenecks. Hence, recent research has shifted toward representation intervention~\cite{yang2025nullu, uppaal2024detox, liu2024reducing}, introducing a dual-stage paradigm. The first step focuses on extracting hallucinatory components from hidden representations by contrasting the model’s internal activations when it produces hallucinated versus faithful outputs. Building on this, the subsequent step updates model parameters to minimize the influence of the extracted hallucinatory components. This paradigm eliminates the dependency on exhaustive dataset reconstruction while enabling real-time hallucination mitigation.

While effective, current representation intervention methods~\cite{yang2025nullu, li2024inference, turner2023steering} suffer from a critical limitation: post-intervened LVLMs would lose general generative capabilities. As quantified in Figure~\ref{fig:motivation} (a), post-intervened LVLMs exhibit more semantic incoherence and lexical repetition rates compared to original models. 
We attribute this limitation to two flaws in current paradigms: First, during the extraction of hallucinatory components, these components are typically highly coupled with general semantic components. Existing methods overlook such shared components, leading to the unintended removal of general semantics. As shown in the left part of Figure~\ref{fig:motivation} (b), the distribution of representations in post-intervened LVLMs diverges significantly from that of the original model.
Second, during parameter updating, substantial perturbations are often applied to all weights within the targeted layers, resulting in the modification of hundreds of millions of parameters, as shown in Figure~\ref{fig:motivation} (c). Such large-scale updates inevitably induce overfitting and severely disrupt the original parameter distributions.

To address these issues, we propose \textbf{MPD}, a \emph{dual-stage framework} that mitigates hallucinations without degrading generative performance. In the first stage, we construct contrastive query pairs using an auxiliary large language model (LLM), where each pair comprises a hallucination-inducing question and a semantically equivalent non-hallucinatory question. By analyzing the model’s representations for both cases, we disentangle hallucination components through orthogonal projection. This design effectively prevents contamination between hallucination and general semantics, as evidenced by the well-aligned post-intervention distributions in Figure~\ref{fig:motivation}(b); theoretical analysis is provided in Section~\ref{sec:analy}. In the second stage, we identify parameters most correlated with hallucination components via cosine similarity and selectively update only these critical parameters. As shown in Figure~\ref{fig:motivation}(c), this selective update substantially reduces parameter perturbations by 42\% on mPLUG-Owl2~\cite{ye2023mplug} and 37\% on MiniGPT4~\cite{zhu2023minigpt}, thus preserving the model’s general capability. We term this approach \textbf{MPD}, highlighting its ability to \underline{m}itigate hallucinations without \underline{p}erformance \underline{d}egradation.

To validate MPD's efficacy, we conduct extensive experiments on five datasets (\textit{i.e.}, CHAIR~\citep{rohrbach2018object}, POPE~\citep{li2023evaluating}, MME~\citep{liu2023improved}, LLaVA-Bench~\citep{liu2023improved}) across three prevailing base LVLMs. Quantitative results demonstrate that MPD achieves advancing hallucination mitigation while preserving 97.4\% of LLMs' generative capabilities. Notably, MPD introduces no extra computational cost during inference compared to the original model.
This breakthrough bridges the efficiency-effectiveness gap in hallucination mitigation, offering practical scalability for real-world LVLM deployment. 

\section{Related Works}
\subsection{Hallucination in LVLMs}
LVLMs integrate powerful LLMs with visual backbones via modality alignment modules~\cite{liu2023visual, zhu2023minigpt, chen2023minigpt, zhu2026guardalign, zhuenhancing}, have demonstrated remarkable performance across various multimodal tasks~\cite{he2024ma, bai2023qwen}. However, these models are prone to object hallucination~\cite{bai2024hallucination}, where the generated response includes non-existent objects or incorrect attributes inconsistent with the visual input~\cite{rohrbach2018object,li2023evaluating}. To mitigate OH, several approaches rely on fine-tuning with additional supervision, including reinforcement learning from human feedback~\cite{sun2023aligning,gunjal2024detecting}, auxiliary annotation~\cite{jiang2024hallucination}, or the inclusion of noisy/negative samples~\cite{liu2023mitigating}. Although effective, such methods require high computational costs and hinder scalability.  To reduce the overhead of fine-tuning, recent studies have explored training-free alternatives. These include decoding-time constraints~\cite{leng2024mitigating,zhang2025self} or post-hoc revision modules~\cite{yin2023woodpecker,chen2024halc} that detect and correct hallucinated tokens using external visual grounding. However, these methods either introduce extra inference latency or depend on auxiliary models. In contrast, our method adopts a training-free, editing-based approach: it identifies hallucination-inducing directions in hidden representations and directly edits the model's internal parameters to suppress them. This enables efficient and consistent hallucination reduction.

\subsection{Subspace Projection and Model Editing}

Subspace projection has emerged as a powerful parameter-editing strategy for steering model behavior, enabling the suppression of undesired semantics (\eg, toxicity, bias, hallucination) by identifying and removing specific directions in the representation or weight space~\cite{turner2023steering, fang2024alphaedit}. In the multimodal context, Nullu~\cite{yang2025nullu} applies this idea to LVLMs by projecting weights onto the null space of hallucination-related directions estimated from hidden features, achieving hallucination suppression via global weight editing. More generally, parameter-modifying model editing techniques adjust internal weights either through meta-learning (\eg, MEND~\cite{MEND}, InstructEdit~\cite{zhang2024instructedit}) or locate-and-edit strategies (\eg, ROME~\cite{meng2022locating}, AnyEdit~\cite{jiang2025anyedit}), typically relying on key-value memory, gradient updates, or causal tracing to identify editable components. Recent extensions to multimodal models include works such as OPERA~\cite{huang2024opera} and VCD~\cite{leng2024mitigating}, which incorporate vision-aware attention refinement or mask-guided cross-modal grounding, often acting at the decoding or cross-attention layers. In contrast, our method introduces a fine-grained orthogonal subspace projection mechanism guided by contrastive pairs between faithful and hallucinated hidden representations. Instead of modifying all parameters, we edit only those weights most aligned with hallucination-specific directions, achieving effective mitigation with minimal disruption to the model's generation behavior.
\section{Method}\label{sec:method}
We first outline hallucination behavior in LVLMs (Section~\ref{sec:preliminary}), then extract hallucinated components via hidden state differencing (Section~\ref{sec:hallu_extract}). Next, we introduce selective parameter editing to suppress hallucination by projecting key weights away from the hallucination subspace (Section~\ref{sec:para_update}), and provide theoretical justification via orthogonal projection onto the faithful subspace (Section~\ref{sec:analy}).

\subsection{Preliminary}\label{sec:preliminary}

A typical LVLM~\cite{zhu2023minigpt,chen2023minigpt,liu2023visual,ye2023mplug} consists of a vision encoder, a text encoder, and a large language model as the decoder. Given an image and a textual input (\eg, a question), the vision encoder extracts visual features, which are projected into the same embedding space as text and concatenated for auto-regressive decoding. However, LVLMs often hallucinate~\cite{yin2023woodpecker,leng2024mitigating}, mentioning objects not present in the image. 
To address this, we build contrastive vision-language pairs sharing the same image but differing in text: one prompt ($x_i^{-}$) includes hallucinated content, while the other ($x_i^{+}$) provides a faithful visual description grounded in the image. The full dataset is denoted as $\mathcal{D} = \{(x_i^{+}, x_i^{-})\}_{i=1}^N$, and the details are provided in Appendix~\ref{sec:implementation}.

\subsection{Hallucination Component Extraction}\label{sec:hallu_extract}
Extracting hallucinatory components is the first step of our method. Since hallucinations arise from subtle deviations in hidden representations between faithful and hallucinated responses, we analyze token-level features across transformer layers \( \ell \in \{L_0, \dots, L\} \) in the base LLM of the LVLM. For each image-caption pair, we obtain token embeddings from both faithful and hallucinated descriptions, denoted as \( \mathbf{x}_{i,\ell}^+ \) and \( \mathbf{x}_{i,\ell}^- \), respectively. To obtain a compact representation, we average token embeddings across the sequence. The resulting vectors are stacked into matrices \( \mathbf{X}_\ell^+, \mathbf{X}_\ell^- \in \mathbb{R}^{N \times D} \), where \( D \) is the feature dimension. We then model the hallucinated representation as a composition of grounded semantics, hallucinatory component, and residual noise:
\begin{equation}
    \X_{\ell}^{-} = \X_{\ell}^{\text{real}} + \X_{\ell}^{\text{hall}} + \boldsymbol{\epsilon}^{-},
\end{equation}
where \( \X_{\ell}^{\text{real}} \) aligns with the image, \( \X_{\ell}^{\text{hall}} \) captures hallucinated semantics, and \( \boldsymbol{\epsilon}^{-} \) is residual noise. This decomposition forms the basis for hallucinatory component extraction.

Since $\X_\ell^{+}$ encodes faithful descriptions, its row space approximates the grounded semantic subspace. We perform singular value decomposition (SVD) on $\X_\ell^{+}$:
\begin{equation}
\X_\ell^{+} = \U_\ell \boldsymbol{\Sigma}_\ell \V_\ell^\top,
\end{equation}
where $\U_\ell \in \mathbb{R}^{N \times C}$ contains the top-$C$ left singular vectors. The projection matrix onto the grounded subspace is given by $\bP_\ell = \U_\ell \U_\ell^\top \in \mathbb{R}^{N \times N}$. We apply $\bP_\ell$ to $\X_\ell^{-}$ to obtain the grounded component $\hat{\X}_\ell^{-} = \bP_\ell \X_\ell^{-}$ and the hallucinatory component is extract as:
\begin{equation}\label{eq:orthogonal_comp}
\tilde{\X}_{\ell} = \X_{\ell}^{-} - \hat{\X}_{\ell}^{-} = (\mathbf{I} - \mathbf{P}_{\ell})\mathbf{X}_{\ell}^{-},
\end{equation}
where $\mathbf{I} \in \mathbb{R}^{D \times D}$ is the identity matrix. The hallucinatory component $\tilde{\X}_{\ell}$ lies orthogonal to the faithful subspace. The justification analysis are provided in Section~\ref{sec:analy}.

\begin{algorithm}[htp]
  \caption{Pipeline of \textbf{\ours}}
  \label{alg}
  \begin{algorithmic}[1]
    \Require Paired data $\mathcal{D}$, target layers $\{\ell\}$ in LVLM.
    \For {$\ell$ in $\{\ell\}$}
      \IState Calculating hidden states of hallucinatory component $\X^{-}$ and faithful components $\X^{+}$.
      \IState Extracting the orthogonal components of the faithful subspace $\Tilde{\X}_{\ell}$ via Eq.~\eqref{eq:orthogonal_comp}.
      \IState Selecting weight vectors $\{\w^{(i)}_{\ell}\}$ with top-$K$ similarity with $\Tilde{\X}_{\ell}$ via Eq.~\eqref{eq:weight_simi}.
      \IState Calculating the projection matrix $\tilde{\Q}_{\ell}$  via {Eq.~\eqref{eq:proj_Q}}.
      \IState Updating the model parameters via $\{\w^{(i)}_{\ell}\}$ via {Eq.~\eqref{eq:updating}}
    \EndFor
  \end{algorithmic}
\end{algorithm}
\vspace{-10pt}



\subsection{Parameters Updating}
\label{sec:para_update}
In Section~\ref{sec:hallu_extract}, we described extracting hallucinatory components from hidden states. We update the parameters most responsible for hallucinations.

\noindent{\textbf{Parameter selection.}} Given the hallucination component $\tilde{\mathbf{X}}_\ell = [\tilde{\mathbf{x}}_{\ell,1}, \dots, \tilde{\mathbf{x}}_{\ell,N}]^\top \in \mathbb{R}^{N \times D}$ extracted at layer $\ell$, we aim to identify the weight parameters that are most responsible for hallucinated generation. Let $\mathbf{W}_\ell \in \mathbb{R}^{L \times D}$ denote the weight matrix at layer $\ell$, where each row $\mathbf{w}_\ell^{(i)}$ is a weight vector associated with a neuron.  To quantify the alignment between each weight and hallucinated semantics, we compute the average cosine similarity between $\mathbf{w}_\ell^{(i)}$ and the $\tilde{\mathbf{x}}_{\ell,j}$:
\begin{equation}\label{eq:weight_simi}
    s_i := \frac{1}{N} \sum_{j=1}^{N} \frac{\text{cos}(\mathbf{w}_\ell^{(i)}, \tilde{\mathbf{x}}_{\ell,j})}{\|\mathbf{w}_\ell^{(i)}\| \cdot \|\tilde{\mathbf{x}}_{\ell,j}\|}.
\end{equation}
We then select the top-$K$ weight vectors with the highest similarity scores $s_i$, and denote their indices by $\mathcal{I}_\ell^{\text{hall}} \subset \{1, \dots, L\}$. These weights are considered most aligned with hallucinated directions and are the targets of our editing step.


\noindent{\textbf{Parameter editing.}}
To suppress hallucination-related influence, we construct a projection operator that eliminates components lying in the hallucination subspace. Following~\cite{fang2024alphaedit}, we define the projection matrix onto the orthogonal complement (null space) of $\tilde{\mathbf{X}}_\ell$ as:
\begin{equation}\label{eq:proj_Q}
\tilde{\mathbf{Q}}_\ell = \mathbf{I} - \tilde{\mathbf{X}}_\ell^\top (\tilde{\mathbf{X}}_\ell \tilde{\mathbf{X}}_\ell^\top)^{-1} \tilde{\mathbf{X}}_\ell.
\end{equation}
This projection operator removes directions spanned by $\tilde{\mathbf{X}}_\ell$. By projecting onto this complement, we can remove the semantic directions associated with hallucinatory component, yielding purified features that better reflect grounded visual evidence. We apply $\tilde{\mathbf{Q}}_\ell$ to selectively edited weights. For each $i \in \mathcal{I}_\ell^{\text{hall}}$, we update the model parameters as follows:
\begin{equation}\label{eq:updating}
    \mathbf{w}_\ell^{(i)} \leftarrow \tilde{\mathbf{Q}}_\ell \mathbf{w}_\ell^{(i)}.
\end{equation}
This targeted projection suppresses hallucination-inducing capacity with minimal impact on unrelated behaviors, as only a small subset of weights is modified. The complete procedure of our $\textbf{\ours}$ is summarized in Algorithm~\ref{alg}.

\subsection{Theoretical Analysis}\label{sec:analy}
Sections~\ref{sec:hallu_extract} and~\ref{sec:para_update} have outlined how to extract hallucination features and suppress them via selective weight projection. We now provide a theoretical justification for using the projection-based residual as a principled estimator of hallucination-specific components.

\noindent\textbf{Justification of projection-based residual for hallucination extraction.} 
To extract the hallucinatory components, we decompose \(\mathbf{X}_\ell^{\text{hall}}\) into a parallel part \(\mathbf{X}_\ell^{\text{hall}, \parallel} = \mathbf{P}_\ell \mathbf{X}_\ell^{\text{hall}}\) within the subspace of faithful representations \(\mathbf{X}_\ell^+\), and an orthogonal part \(\mathbf{X}_\ell^{\text{hall}, \perp} = (\mathbf{I} - \mathbf{P}_\ell) \mathbf{X}_\ell^{\text{hall}}\), corresponding to disentangled hallucinatory components. We define \(\tilde{\mathbf{X}}_\ell = (\mathbf{I} - \mathbf{P}_\ell) \mathbf{X}_\ell^-\) as the extracted hallucinatory components from the response representation \(\mathbf{X}_\ell^-\). As a comparison, we consider the naive difference-based estimate \(\tilde{\mathbf{X}}_\ell^{\text{diff}} = \mathbf{X}_\ell^- - \mathbf{X}_\ell^+\), which entangles hallucination with shared semantics.

\begin{proposition}\label{prop:projection_residual}

For any \(\mathbf{X}_\ell^{\text{hall}, \parallel} = \mathbf{P}_\ell \mathbf{X}_\ell^{\text{hall}}\), the extracted hallucinatory components \(\tilde{\mathbf{X}}_\ell = (\mathbf{I} - \mathbf{P}_\ell) \mathbf{X}_\ell^-\) provide a more accurate estimation of the disentangled hallucinatory signal \(\mathbf{X}_\ell^{\text{hall}, \perp}\) than the naive difference-based estimate \(\tilde{\mathbf{X}}_\ell^{\text{diff}} = \mathbf{X}_\ell^- - \mathbf{X}_\ell^+\), in terms of expected squared error:
\[
\mathbb{E} \|\tilde{\mathbf{X}}_\ell - \mathbf{X}_\ell^{\text{hall}, \perp}\|_F^2 \leq \mathbb{E} \|\tilde{\mathbf{X}}_\ell^{\text{diff}} - \mathbf{X}_\ell^{\text{hall}, \perp}\|_F^2.
\]
\end{proposition}

\begin{proof}

The extracted hallucinatory components is \(\tilde{\mathbf{X}}_\ell = (\mathbf{I} - \mathbf{P}_\ell) \mathbf{X}_\ell^-\), where \(\mathbf{P}_\ell\) projects onto the faithful subspace spanned by \(\mathbf{X}_\ell^+\). For \(\mathbf{X}_\ell^- = \mathbf{X}_\ell^{\text{real}} + \mathbf{X}_\ell^{\text{hall}} + \boldsymbol{\epsilon}^{-}\), with assuming \(\boldsymbol{\epsilon}^{-} \sim \mathcal{N}(0, \sigma_{-}^2 \mathbf{I})\), and \(\mathbf{X}_\ell^{\text{real}}\) lies in the faithful subspace (\(\mathbf{P}_\ell \mathbf{X}_\ell^{\text{real}} \approx \mathbf{X}_\ell^{\text{real}}\)), the error is:
\[
\tilde{\mathbf{X}}_\ell - \mathbf{X}_\ell^{\text{hall}, \perp} \approx (\mathbf{I} - \mathbf{P}_\ell) \boldsymbol{\epsilon}^{-},
\]
yielding:
\[
\mathbb{E} \|\tilde{\mathbf{X}}_\ell - \mathbf{X}_\ell^{\text{hall}, \perp}\|_F^2 = \sigma_{-}^2 (D - C) N.
\]
The difference-based residual is \(\tilde{\mathbf{X}}_\ell^{\text{diff}} = \mathbf{X}_\ell^- - \mathbf{X}_\ell^+\), with \(\mathbf{X}_\ell^+ = \mathbf{X}_\ell^{\text{real}} + \boldsymbol{\epsilon}^{+}\), \(\boldsymbol{\epsilon}^{+} \sim \mathcal{N}(0, \sigma_{+}^2 \mathbf{I})\). The error is:
\[
\tilde{\mathbf{X}}_\ell^{\text{diff}} - \mathbf{X}_\ell^{\text{hall}, \perp} = \mathbf{X}_\ell^{\text{hall}, \parallel} + \boldsymbol{\epsilon}^{-} - \boldsymbol{\epsilon}^{+}.
\]
Based on these, we have:
\[
\mathbb{E} \|\tilde{\mathbf{X}}_\ell^{\text{diff}} - \mathbf{X}_\ell^{\text{hall}, \perp}\|_F^2 = \|\mathbf{X}_\ell^{\text{hall}, \parallel}\|_F^2 + \sigma_{-}^2 D N + \sigma_{+}^2 D N.
\]
Since \(\sigma_{-}^2 (D - C) N < \sigma_{-}^2 D N\), and \(\|\mathbf{X}_\ell^{\text{hall}, \parallel}\|_F^2 \geq 0\), \(\sigma_{+}^2 \geq 0\) contribute additional non-negative terms, we have:
\[
\sigma_{-}^2 (D - C) N < \|\mathbf{X}_\ell^{\text{hall}, \parallel}\|_F^2 + \sigma_{-}^2 D N + \sigma_{+}^2D N.
\]
Thus, $\tilde{\mathbf{X}}_{\ell}$ provides a more accurate estimation. See Appendix~\ref{appendix:projection_proof} for details.
\end{proof}
\vspace{-8pt}

\begin{table*}[!hbpt]
\centering
\caption{CHAIR evaluation results on MSCOCO dataset of LVLMs (MiniGPT-4, mPLUG-Owl2, and LLaVA-1.5-7B) with different methods for mitigating OH. Lower $\text{CHAIR}_S$ and $\text{CHAIR}_I$ indicate less OH. Higher BLEU generally represents higher captioning quality. We use 64 as the max token number in this experiment. Bold indicates the best result of all methods.}
\tabstyle{1.5pt}
\renewcommand{\arraystretch}{1.2} 
\begin{tabular}{ l | c c c | c c c | c c c}
\toprule
\multirow{3}{*}{{Method}} 
&\multicolumn{3}{c|}{MiniGPT-4}
&\multicolumn{3}{c|}{mPLUG-Owl2}
&\multicolumn{3}{c}{LLaVA-1.5-7B} \\
\cmidrule{2-10}
& CHAIR$_S$$\downarrow$  & CHAIR$_I$$\downarrow$ & BLEU$\uparrow$
& CHAIR$_S$$\downarrow$  & CHAIR$_I$$\downarrow$ & BLEU$\uparrow$
& CHAIR$_S$$\downarrow$  & CHAIR$_I$$\downarrow$ & BLEU$\uparrow$
   \\ 
\midrule
Greedy
&$\text{32.40}_{\pm \text{2.20}}$ &$\text{12.20}_{\pm \text{0.42}}$ &$\text{14.57}_{\pm \text{0.11}}$ 
&$\text{22.90}_{\pm \text{0.90}}$ &$\text{8.62}_{\pm \text{0.11}}$ &$\text{15.01}_{\pm \text{0.24}}$  
&$\text{20.40}_{\pm \text{2.80}}$ &$\text{7.08}_{\pm \text{0.33}}$ &$\text{15.72}_{\pm \text{0.10}}$ 
\\
DoLa
&$\text{31.90}_{\pm \text{3.30}}$ &$\text{12.15}_{\pm \text{0.89}}$ &$\text{14.54}_{\pm \text{0.12}}$ 
&$\text{22.40}_{\pm \text{1.80}}$ &$\text{8.36}_{\pm \text{0.04}}$ &$\text{15.13}_{\pm \text{0.21}}$  
&$\text{20.20}_{\pm \text{2.80}}$ &$\text{6.75}_{\pm \text{0.54}}$ &$\text{15.68}_{\pm \text{0.10}}$ 
\\ 
OPERA
&$\text{29.70}_{\pm \text{0.30}}$ &$\text{11.96}_{\pm \text{0.29}}$ &$\text{14.82}_{\pm \text{0.05}}$ 
&$\text{20.07}_{\pm \text{2.07}}$ &$\text{7.18}_{\pm \text{0.39}}$ &$\text{15.41}_{\pm \text{0.12}}$ 
&$\text{17.50}_{\pm \text{0.50}}$ &$\text{6.07}_{\pm \text{0.32}}$ &$\text{16.02}_{\pm \text{0.02}}$ 
\\ 
VCD
&$\text{29.00}_{\pm \text{2.80}}$ &$\text{12.64}_{\pm \text{1.19}}$ &$\text{14.42}_{\pm \text{0.01}}$ 
&$\text{22.80}_{\pm \text{0.80}}$ &$\text{8.68}_{\pm \text{0.17}}$ &$\text{15.14}_{\pm \text{0.13}}$  
&$\text{20.30}_{\pm \text{1.10}}$ &$\text{7.28}_{\pm \text{0.10}}$ &$\text{14.53}_{\pm \text{0.01}}$ 
\\
LURE
&$\text{27.88}_{\pm \text{2.25}}$ &$\text{10.20}_{\pm \text{0.85}}$ &$\text{15.03}_{\pm \text{0.11}}$ 
&$\text{21.27}_{\pm \text{0.06}}$ &$\text{7.67}_{\pm \text{0.16}}$ &$\text{15.65}_{\pm \text{0.15}}$ 
&$\text{19.48}_{\pm \text{2.35}}$ &$\text{6.50}_{\pm \text{0.38}}$ &$\text{15.97}_{\pm \text{0.01}}$   
\\ 
HALC
&$\text{25.20}_{\pm \text{2.00}}$ &$\text{9.42}_{\pm \text{0.41}}$ &$\text{14.91}_{\pm \text{0.13}}$ 
&$\text{18.80}_{\pm \text{1.20}}$ &$\text{7.00}_{\pm \text{0.01}}$ &$\text{15.33}_{\pm \text{0.24}}$ 
&$\text{16.90}_{\pm \text{2.10}}$ &$\text{5.72}_{\pm \text{0.55}}$ &$\text{16.02}_{\pm \text{0.04}}$ 
\\ 
Nullu 
&$\text{{21.40}}_{\pm \text{1.20}}$ &$\text{{8.99}}_{\pm \text{0.56}}$ &$\text{14.81}_{\pm \text{0.06}}$ 
&$\text{{15.60}}_{\pm \text{1.50}}$ &$\text{{5.77}}_{\pm \text{0.011}}$ &$\text{15.45}_{\pm \text{0.13}}$ 
&$\text{{15.20}}_{\pm \text{0.50}}$ &$\text{{5.30}}_{\pm \text{0.13}}$ &$\text{15.69}_{\pm \text{0.07}}$ 
\\ \midrule
\rowcolor{mygray}
\textbf{\ours} 
&$\text{\textbf{19.40}}_{\pm \text{1.00}}$ &$\text{\textbf{7.50}}_{\pm \text{0.36}}$ &$\textbf{14.98}_{\pm \text{0.06}}$ 
&$\text{\textbf{14.00}}_{\pm \text{1.20}}$ &$\text{\textbf{4.99}}_{\pm \text{0.01}}$ &$\textbf{16.06}_{\pm \text{0.01}}$ 
&$\text{\textbf{12.80}}_{\pm \text{0.60}}$ &$\text{\textbf{4.20}}_{\pm \text{0.03}}$ &$\text{15.31}_{\pm \text{0.04}}$ 
\\ 
\bottomrule
\end{tabular}
\label{tab:chair_results}
\end{table*}

\begin{table*}[!ht]
  \centering
  \caption{
    Results on POPE. “Original” denotes direct inference using the original LVLMs, 
    while “Nullu” refers to models enhanced with our proposed editing method.
  }
    \tabstyle{8pt}
    \renewcommand{\arraystretch}{0.9} 
  \begin{tabular}{ccccccc}
    \toprule
    {Setting} & {Model} & {Method} & {Accuracy} & {Precision} & {Recall} & {F{1} Score} \\
    \midrule

    \multirow{10}{*}{\textit{random}} 
        & \multirow{3}{*}{MiniGPT4} 
            & Original & 57.33 & 53.66 & 97.13 & 70.04 \\
        &   & Nullu & 58.00 & 54.41 & 98.53 & 70.14 \\
        &   & \cellcolor{black!10}\textbf{\ours} & \cellcolor{black!10}\textbf{59.39} & \cellcolor{black!10}\textbf{55.67} & \cellcolor{black!10}\textbf{98.86} & \cellcolor{black!10}\textbf{71.81} \\
    \cmidrule{2-7}

        & \multirow{3}{*}{mPLUG-Owl2} 
            & Original & 81.83 & 77.80 & 89.07 & 83.06 \\
        &   & Nullu & 83.33 & 79.10 & 90.60 &84.46 \\
        &   & \cellcolor{black!10}\textbf{\ours} & \cellcolor{black!10}\textbf{85.63} & \cellcolor{black!10}\textbf{82.30} & \cellcolor{black!10}\textbf{91.56} & \cellcolor{black!10}\textbf{85.76} \\
    \cmidrule{2-7}
        
        & \multirow{3}{*}{LLaVA-1.5-7B} 
            & Original & 87.98 & 88.55 & 80.43 & 86.03 \\
        &   & Nullu & 88.23 & 91.31 & 81.86 & 87.43 \\
        &   & \cellcolor{black!10}\textbf{\ours} & \cellcolor{black!10}\textbf{89.31} & \cellcolor{black!10}\textbf{92.41} & \cellcolor{black!10}\textbf{84.10} & \cellcolor{black!10}\textbf{88.20} \\
    \midrule

    \multirow{10}{*}{\textit{popular}} 
        & \multirow{3}{*}{MiniGPT4} 
            & Original & 56.63 & 53.66 & 97.13 & 69.13 \\
        &   & Nullu & 52.66 & 51.38 & 98.53 & 67.57 \\
        &   & \cellcolor{black!10}\textbf{\ours} & \cellcolor{black!10}\textbf{53.96} & \cellcolor{black!10}\textbf{52.24} & \cellcolor{black!10}\textbf{98.75} & \cellcolor{black!10}\textbf{68.53} \\
    \cmidrule{2-7}

        & \multirow{3}{*}{mPLUG-Owl2} 
            & Original & 75.77 & 70.35 & 89.07 & 78.61 \\
        &   & Nullu & 77.47 & 71.75 & 90.60 & 80.08 \\
        &   & \cellcolor{black!10}\textbf{\ours} & \cellcolor{black!10}\textbf{79.57} & \cellcolor{black!10}\textbf{74.79} & \cellcolor{black!10}\textbf{91.80} & \cellcolor{black!10}\textbf{81.58} \\
    \cmidrule{2-7}
        
        & \multirow{3}{*}{LLaVA-1.5-7B} 
            & Original & 84.68 & 81.51 & 81.43 & 83.32 \\
        &   & Nullu & 85.35 & 84.35 & 82.13 & 84.93 \\
        &   & \cellcolor{black!10}\textbf{\ours} & \cellcolor{black!10}\textbf{86.37} & \cellcolor{black!10}\textbf{85.82} & \cellcolor{black!10}\textbf{83.93} & \cellcolor{black!10}\textbf{85.52} \\
    \midrule

    \multirow{10}{*}{\textit{adversarial}} 
        & \multirow{3}{*}{MiniGPT4} 
            & Original & 50.17 & 50.21 & 97.13 & 67.02 \\
        &   & Nullu & 51.90 & 50.98 & 97.86 & 67.04 \\
        &   & \cellcolor{black!10}\textbf{\ours} & \cellcolor{black!10}\textbf{52.83} & \cellcolor{black!10}\textbf{51.50} & \cellcolor{black!10}\textbf{98.53} & \cellcolor{black!10}\textbf{68.22} \\
    \cmidrule{2-7}

        & \multirow{3}{*}{mPLUG-Owl2} 
            & Original & 72.77 & 67.17 & 89.07 & 76.58 \\
        &   & Nullu & 73.23 & 67.25 & 90.40 & 77.15 \\
        &   & \cellcolor{black!10}\textbf{\ours} & \cellcolor{black!10}\textbf{75.23} & \cellcolor{black!10}\textbf{69.15} & \cellcolor{black!10}\textbf{91.60} & \cellcolor{black!10}\textbf{78.72} \\
    \cmidrule{2-7}
        
        & \multirow{3}{*}{LLaVA-1.5-7B} 
            & Original & 77.97 & 72.79 & 79.43 & 78.24 \\
        &   & Nullu & 78.61 & 75.51 & 79.86 & 79.07 \\
        &   & \cellcolor{black!10}\textbf{\ours} & \cellcolor{black!10}\textbf{79.56} & \cellcolor{black!10}\textbf{75.98} & \cellcolor{black!10}\textbf{82.26} & \cellcolor{black!10}\textbf{80.18} \\
    \bottomrule
    \end{tabular}
    \vspace{-10pt}
  \label{tab:pope}
\end{table*}

\section{Experiments}
In this section, we conduct experiments to address the following research questions (RQ):

\begin{itemize}[leftmargin=*]
    \item \textbf{RQ1:} Does MPD effectively mitigate object hallucinations in LVLMs across multiple settings, including sentence-level, instance-level, and representation-level metrics?
        
    \item \textbf{RQ2:} Does our method maintain the perception and reasoning abilities of LVLMs while suppressing hallucinations? 

    \item \textbf{RQ3}: How do the edited layer ranges, sample sizes, retained SVD components, and selected weights affect the hallucination suppression?

\end{itemize}

\subsection{Experimental Setup}\label{sec:exp}

We begin by outlining the evaluation benchmarks, metrics, and baseline methods used to assess our approach across multiple LVLMs.

\noindent{\textbf{Base models and baselines.}}  
Our experiments are conducted on three representative LVLMs: \textbf{MiniGPT-4 V2}~\citep{zhu2023minigpt}, \textbf{LLaVA-1.5-7B}~\citep{liu2023visual}, and \textbf{mPLUG-Owl2}~\citep{ye2023mplug}. For comparison, we include several state-of-the-art object hallucination (OH) mitigation methods: DoLa~\cite{chuang2023dola}, OPERA~\cite{huang2024opera}, VCD~\cite{leng2024mitigating}, LURE~\cite{zhou2023analyzing}, HALC~\cite{chen2024halc}, and Nullu~\cite{yang2025nullu}.  We use LURE~\cite{zhou2023analyzing} as the paired data following the previous work~\cite{yang2025nullu}, and prompt all methods with ``\textit{Please describe this image in detail.}''. 

\noindent{\textbf{Evaluation benchmarks and metrics.}}  
We evaluate OH mitigation performance using the following three categories of benchmarks. \textbf{MSCOCO-based metrics.} We adopt CHAIR~\citep{rohrbach2018object} and POPE~\citep{li2023evaluating} on the MSCOCO~\cite{lin2014microsoft} dataset to evaluate hallucination in image descriptions. CHAIR includes sentence-level ($\text{CHAIR}_S$) and instance-level ($\text{CHAIR}_I$) hallucination metrics, with lower scores indicating better grounding. POPE measures LVLMs' ability to correctly answer yes/no object presence queries under different negative sampling strategies. \textbf{MME.} The Multimodal Large Language Model Evaluation (MME) benchmark~\citep{fu2023mme} is used to assess perception and reasoning abilities across 14 tasks. We follow prior work~\citep{yin2023woodpecker, huang2024opera, chen2024halc} and focus on four hallucination-relevant subsets: ``Existence'', ``Count'', ``Position'', and ``Color''. \textbf{LLaVA-Bench.} The LLaVA-Bench~\citep{liu2023improved} benchmark comprises human-curated images and a diverse set of questions to evaluate open-ended captioning, reasoning, and conversational abilities. Details are provided in Appendix~\ref{sec:implementation}.

    


\subsection{Performance on Object Hallucinations (RQ1)}
To evaluate MPD’s effectiveness in mitigating object hallucinations, we conduct experiments on CHAIR and POPE, followed by an analysis of hidden representation alignment. 

\noindent{\textbf{Results on CHAIR.}} 
We evaluate the effectiveness of our method on object hallucination using the CHAIR benchmark, which includes both sentence-level (CHAIR$_{S}$) and instance-level (CHAIR$_{I}$) hallucination metrics, as well as BLEU for caption quality. Table~\ref{tab:chair_results} summarizes the performance across three representative LVLMs.

\begin{itemize}[leftmargin=*]
    \item \textbf{Obs 1: MPD significantly reduces sentence-level hallucinations while preserving caption quality.}  
    On all models, our method achieves the lowest CHAIR$_{S}$ scores: 12.80, 19.40, and 14.00, outperforming strong baselines such as Nullu (15.20, 20.30, 15.30). BLEU scores remain competitive, with the highest score on mPLUG-Owl2 (16.06) and solid results on the other two, indicating reduced hallucination without sacrificing generation quality.

    \item \textbf{Obs 2: There is improvement with MPD in fine-grained hallucination mitigation.}  
    In terms of CHAIR$_{I}$, our method outperforms DoLa, HALC, and Nullu across all models. This confirms its strength in reducing object-level hallucinations while maintaining descriptive relevance.
\end{itemize}

\noindent{\textbf{Results on POPE.}}  
We evaluate our method on the POPE benchmark~\cite{li2023evaluating}, which measures hallucination robustness under three query sampling strategies: \textit{random}, \textit{popular}, and \textit{adversarial}. Table~\ref{tab:pope} reports results on LLaVA-1.5-7B, MiniGPT-4, and mPLUG-Owl2. We also report results on OPOPE~\cite{li2023evaluating}, AMBER~\cite{wang2023amber}, and MMHalBench~\cite{sun2023aligning}. The results are provided in Appendix~\ref{sec:more_res}.

\begin{itemize}[leftmargin=*]
    \item \textbf{Obs 3: MPD consistently achieves the highest F{1} scores across all settings.}  
Under all sampling strategies, our method outperforms original models and the Nullu baseline. For example, in the \textit{random} setting, we achieve 88.20 (LLaVA-1.5-7B), 85.76 (mPLUG-Owl2), and 71.81 (MiniGPT-4), with similar gains in \textit{popular} and \textit{adversarial} settings. These results confirm our approach’s robustness and generalizability under varied hallucination conditions.

    \item \textbf{Obs 4: Complex visual scenes do not hinder the performance of MPD.}  
    Even in adversarial settings with distractors, our method maintains high accuracy. On mPLUG-Owl2, we reach 78.72 F$1$ score, outperforming original (76.58) and Nullu (77.15), demonstrating effectiveness.
\end{itemize}

\begin{table}[!t]
    \centering
    \tabstyle{5pt}
    \caption{Evaluation Results on HallusionBench following the setting in ~\cite{HallusionBench}.}
    \label{tab:hallusionbench}
    \begin{tabular}{llllll}
        \toprule
        Model & fACC & qACC & easyA & hardA & aACC \\
        \midrule
        LLaVA-1.5-7B & 17.9 & 8.1 & 36.0 & 36.7 & 41.5 \\
        Nullu & 18.7 & 8.6 & 36.6 & 37.4 & 44.2 \\
        \textbf{\ours} & 19.9 & 9.0 & 37.3 & 38.1 & 44.3 \\
        \midrule
        MiniGPT-4 & 10.1 & 8.7 & 31.8 & 27.6 & 35.7 \\
        Nullu & 10.7 & 9.5 & 32.3 & 28.1 & 36.8 \\
        \textbf{\ours} & 11.5 & 10.2 & 33.1 & 28.7 & 37.3 \\
        \midrule
        mPLUG-Owl2 & 19.9 & 13.8 & 44.8 & 39.1 & 47.3 \\
        Nullu & 10.7 & 9.5 & 32.3 & 28.1 & 36.8 \\
        \textbf{\ours} & 11.5 & 10.2 & 33.1 & 28.7 & 37.3 \\
        \bottomrule
    \end{tabular}
    \vspace{-10pt}
\end{table}
\noindent{\textbf{Results on Hallusionbench.}}
HallusionBench dataset~\cite{HallusionBench} evaluates fine-grained hallucinations beyond object existence.
As shown in Table~\ref{tab:hallusionbench}, \ours consistently improves all metrics over the base models and Nullu across three LVLM backbones, including gains on challenging subsets such as \textit{hardA} and \textit{aACC}.
These results indicate that our method generalizes beyond object-centric hallucinations.
Further evidence on broader hallucination benchmarks is provided in Appendix~\ref{sec:more_res}, as shown in Table~\ref{tab:non-object}.

\begin{table}[!ht]
    \centering
    \tabstyle{5pt}
    \caption{Results on LLaVA-Bench following the setting in~\cite{leng2024mitigating}. Both metrics are on a scale of 10.}
    \renewcommand{\arraystretch}{1.1}
    \begin{tabular}{lccc}
        \toprule
        {Model} & {Method} & {Accuracy} & {Detailedness} \\
        \midrule
        \multirow{2}{*}{MiniGPT-4} & Original & 4.05 & 3.95 \\
        & \textbf{\ours} & \textbf{5.53} & \textbf{4.67} \\
        \midrule
        \multirow{2}{*}{mPLUG-Owl2} & Original & 5.76 & 4.22 \\
        & \textbf{\ours} & \textbf{6.13} & \textbf{4.62} \\
        \midrule
        \multirow{2}{*}{LLaVA-1.5-7B} & Original & 5.59 & 4.72 \\
        & \textbf{\ours} & \textbf{6.39} & \textbf{5.52} \\
        \bottomrule
    \end{tabular}
    \label{tab:gpt4v}
    \vspace{-10pt}
\end{table}

\subsection{Performance on Generative Capacity (RQ2)}\label{sec:gpt4v_eval}
To assess the general generative capacity of MPD beyond structured object grounding, we evaluate it on LLaVA-Bench and MME, which cover open-ended and structured settings~\cite{leng2024mitigating}. More results are provided in Appendix~\ref{sec:more_res}.

\noindent{\textbf{Results on LLaVA-Bench.}}  
Table~\ref{tab:gpt4v} reports GPT-4V-evaluated scores on \textit{Accuracy} and \textit{Detailedness}. MPD consistently improves both dimensions across all models. On LLaVA-1.5-7B, it raises accuracy from 5.59 to 6.39 and detailedness from 4.72 to 5.52, indicating enhanced factual alignment and richer responses without harming fluency.

\begin{itemize}[leftmargin=*]
    \item \textbf{Obs 5: Both grounding and response quality are improved by MPD.}  
    The gains in both accuracy and detailedness suggest that MPD effectively reduces hallucinations while preserving the model’s generative expressiveness.
\end{itemize}

\begin{figure}[!t]
    \centering
    \includegraphics[width=0.8\linewidth]{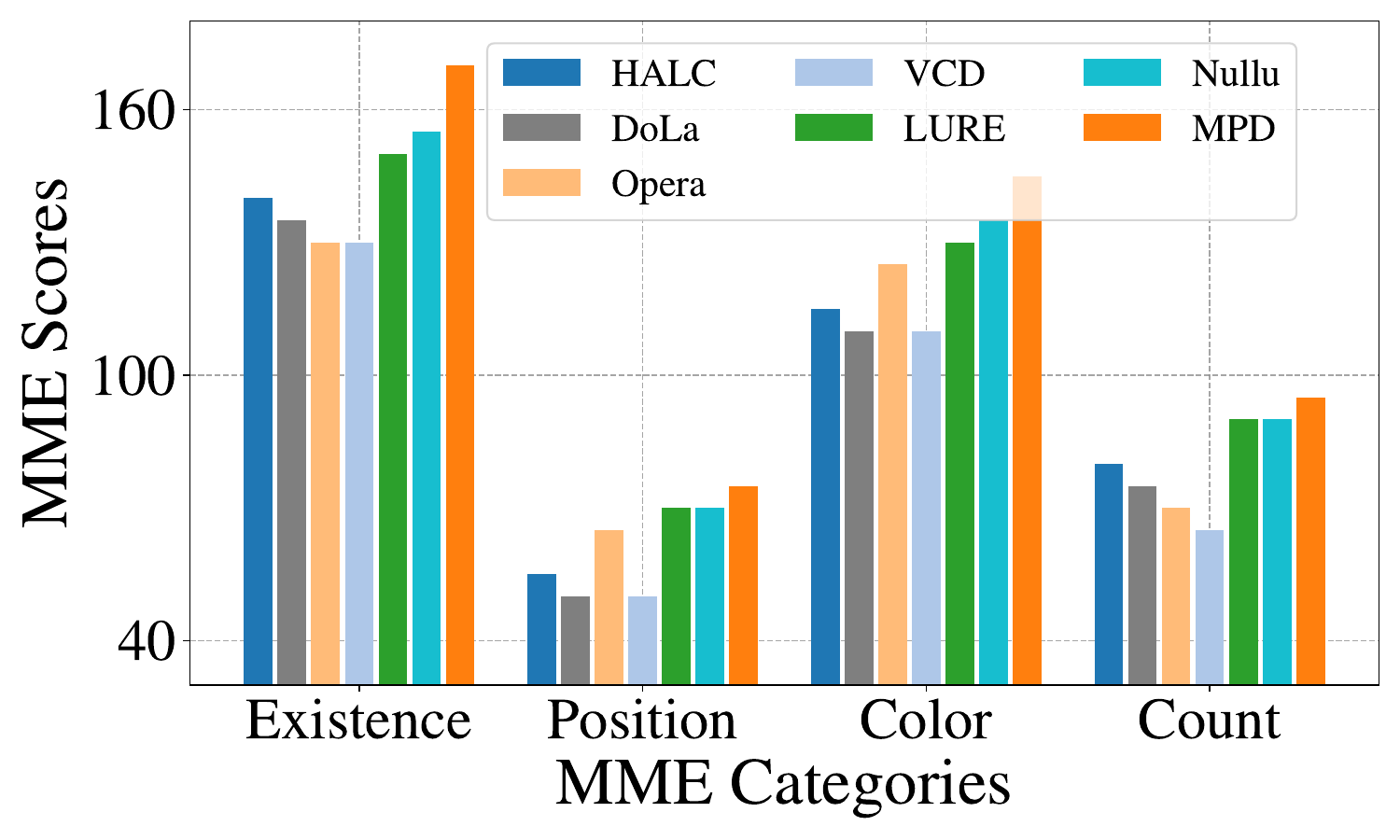}
    \vspace{-10pt}
    \caption{Comparisons of MME scores.}
    \label{fig:mme}
\end{figure}
\noindent{\textbf{Results on MME.}}  
Figure~\ref{fig:mme} summarizes results for four hallucination-sensitive subsets. MPD achieves the best performance on \textit{Existence} and \textit{Count}, measuring object hallucination, and remains competitive on \textit{Position} and \textit{Color}.
\begin{itemize}[leftmargin=*]
    \item \textbf{Obs 6: MPD strengthens robustness under structured hallucination evaluation.}  
    The improvements on \textit{Existence} and \textit{Count} confirm that MPD enhances grounding accuracy in perception-oriented tasks, further supporting its generalization across evaluation paradigms.
\end{itemize}

\begin{figure*}[!ht]
    \centering
    \includegraphics[width=0.85\linewidth]{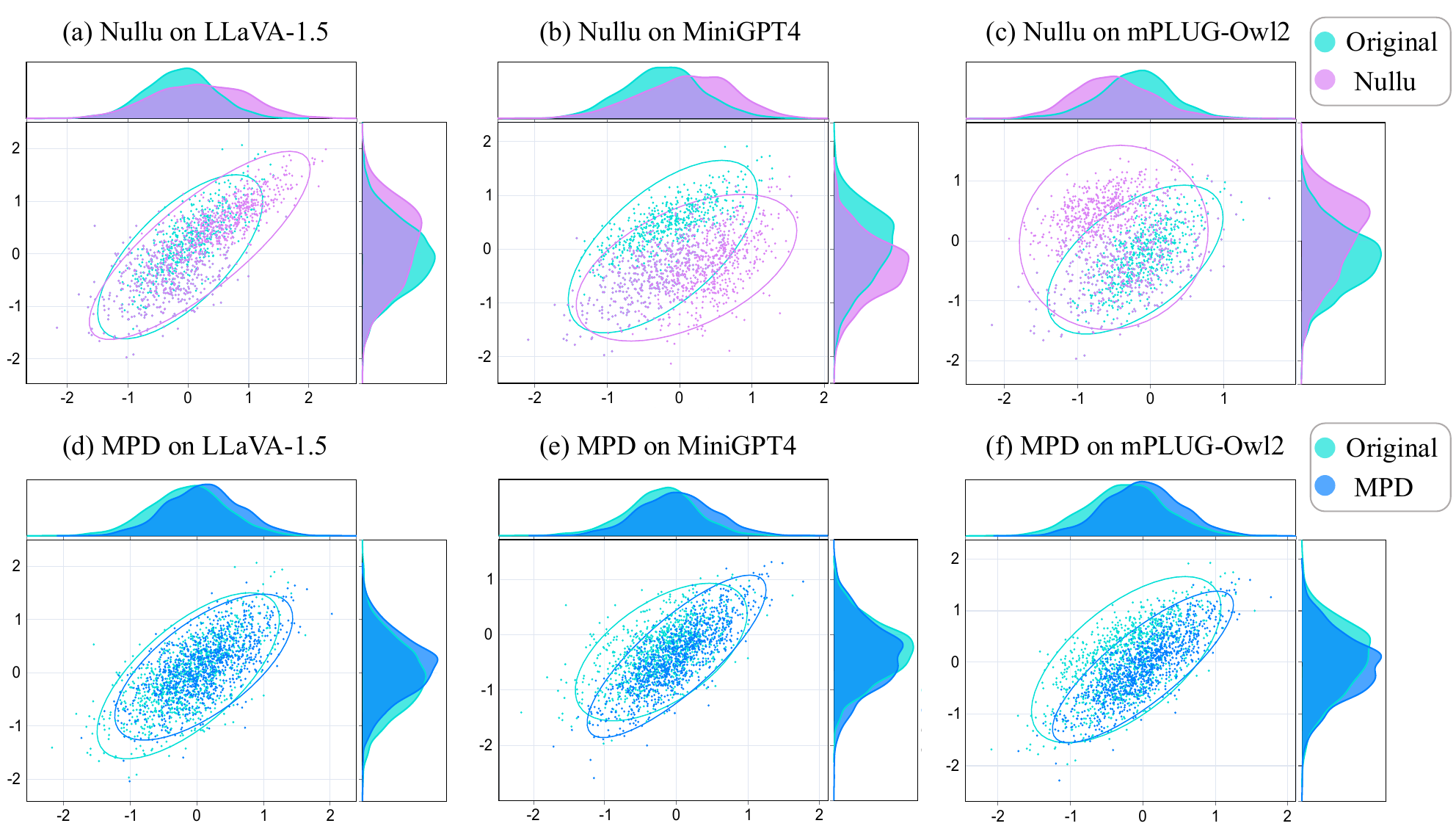}
    \caption{Distribution of hidden representations before and after editing across different LVLMs. Each subplot visualizes the principal components of token embeddings using PCA, comparing the original model and the edited version. The top and right marginal curves show the distributions along each principal axis. Compared to Nullu, our method (MPD) induces smaller representation shifts while achieving more effective hallucination suppression.}
    \label{fig:distribution}
\end{figure*}

\noindent{\textbf{Analysis of hidden representations.}} To assess the impact of editing on internal representations, we examine whether our method preserves the distributional structure of hidden features. We sample 1,000 faithful prompts and extract their token-level representations from the original LVLMs. After applying edits via Nullu and MPD, we re-extract the features, project them into two dimensions using PCA, and visualize them along with their marginal distributions. As shown in Figure~\ref{fig:distribution}, MPD maintains close alignment with the original features, while Nullu introduces noticeable shifts.

\begin{table}[!t]
    \centering
    \tabstyle{5pt}
    \caption{Analysis of editing different layers.}
    \renewcommand{\arraystretch}{1.2}
    \begin{tabular}{l|ccc}
        \toprule
        $\{\ell\}$ &{CHAIR}$_S\downarrow$ &{CHAIR}$_I\downarrow$ &BLEU$\uparrow$ \\ 
        \midrule
        16-32 & \textbf{12.80} & \textbf{4.20} & \textbf{15.31} \\  
        18-32 & 13.20 & 5.31 & 15.00 \\ 
        20-32 & 13.50 & 5.52 & 14.97 \\
        22-32 & 14.03 & 5.95 & 14.80 \\ 
        24-32 & 14.97 & 6.41 & 14.73 \\ 
        26-32 & 16.95 & 6.87 & 14.72 \\ 
        28-32 & 18.41 & 7.32 & 14.70 \\ 
        30-32 & 20.10 & 7.75 & 14.70 \\ 
        \bottomrule
    \end{tabular}
    \vspace{-10pt}
    \label{tab:ablation_layers}
\end{table}

\subsection{Ablation Study (RQ3)}
To analyze the contribution of each module in MPD, we perform ablation studies on editing layers and retained subspace dimensions using LLaVA-1.5-7B across CHAIR, POPE, and MME. More ablations are provided in Appendix~\ref{sec:abla_size} and ~\ref{sec:abla_weights}.

\noindent{\textbf{Impact of editing layers.}} To assess the impact of editing layers, we perform an ablation study on LLaVA-1.5-7B using CHAIR$_S$, CHAIR$_I$, and BLEU metrics, as shown in Table~\ref{tab:ablation_layers}.

\begin{itemize}[leftmargin=*]
    \item \textbf{Obs 7: Deeper edits yield hallucination suppression.}  
    Editing from deeper layers (\eg, 22–32 or 24–32) improves CHAIR scores, with the best result achieved at 16–32 (CHAIR$_S$ = 12.80, CHAIR$_I$ = 4.20). In contrast, narrow or shallow ranges (\eg, 30–32 or 26–32) result in weaker performance and no BLEU gain, indicating insufficient removal of hallucinated components.
\end{itemize}

\section{Conclusion}\label{sec:limitions}
We present MPD, an optimization framework that mitigates hallucinations in LVLMs without degrading generative capacity. It disentangles hallucinatory components via orthogonal projection on contrastive representations and selectively updates parameters most correlated with these components using cosine similarity, thus avoiding large-scale perturbations. Extensive evaluations across multiple LVLMs and benchmarks show that MPD suppresses hallucinations while retaining 97.4\% of generative capacity with no added inference cost.

\section*{Limitations}
While MPD improves alignment on current benchmarks, it does not address hallucinations 
arising from broader data biases or prompt limitations. Moreover, automated benchmarks 
may not fully capture long-form coherence or stylistic diversity. Its robustness under 
highly ambiguous visual inputs remains to be evaluated.

\section*{Acknowledgements}
This research was supported by the National Natural Science Foundation of China (U24B20180, 62525211, 62576330) and the Natural Science Foundation of Anhui Province (2508085MF143).
\bibliography{custom}

\newpage
\appendix


\section{Proof of Proposition~\ref{prop:projection_residual}}
\label{appendix:projection_proof}

We provide a detailed proof that the extracted hallucinatory component \(\tilde{\mathbf{X}}_\ell = (\mathbf{I} - \mathbf{P}_\ell) \mathbf{X}_\ell^-\) is more accurate than the difference-based residual \(\tilde{\mathbf{X}}_\ell^{\text{diff}} = \mathbf{X}_\ell^- - \mathbf{X}_\ell^+\) in estimating the orthogonal component of the hallucinated semantics \(\mathbf{X}_\ell^{\text{hall}, \perp}\), without assuming that the parallel component \(\mathbf{X}_\ell^{\text{hall}, \parallel} = \mathbf{P}_\ell \mathbf{X}_\ell^{\text{hall}}\) is small.

Given:
\begin{itemize}
    \item \(\mathbf{X}_\ell^+ \in \mathbb{R}^{N \times D}\): Faithful representations, spanning the faithful subspace.
    \item \(\mathbf{X}_\ell^- = \mathbf{X}_\ell^{\text{real}} + \mathbf{X}_\ell^{\text{hall}} + \boldsymbol{\epsilon}^{-}\): Hallucinated representations, composed of real semantics, hallucinated components, and Gaussian noise \(\boldsymbol{\epsilon}^{-} \sim \mathcal{N}(0, \sigma_{-}^2 \mathbf{I})\).
    \item \(\mathbf{P}_\ell = \mathbf{U}_\ell \mathbf{U}_\ell^\top\): Orthogonal projection onto the subspace spanned by top-\(C\) singular vectors \(\mathbf{U}_\ell\) from the SVD of \(\mathbf{X}_\ell^+\).
    \item \(\mathbf{X}_\ell^{\text{hall}} = \mathbf{X}_\ell^{\text{hall}, \parallel} + \mathbf{X}_\ell^{\text{hall}, \perp}\): Decomposition into in-subspace and orthogonal components w.r.t. \(\mathbf{P}_\ell\).
\end{itemize}

The goal is to show:
\[
\mathbb{E} \|\tilde{\mathbf{X}}_\ell - \mathbf{X}_\ell^{\text{hall}, \perp}\|_F^2 \leq \mathbb{E} \|\tilde{\mathbf{X}}_\ell^{\text{diff}} - \mathbf{X}_\ell^{\text{hall}, \perp}\|_F^2.
\]

The projection-based residual is:
\[
\tilde{\mathbf{X}}_\ell = (\mathbf{I} - \mathbf{P}_\ell) \mathbf{X}_\ell^- = (\mathbf{I} - \mathbf{P}_\ell) (\mathbf{X}_\ell^{\text{real}} + \mathbf{X}_\ell^{\text{hall}} + \boldsymbol{\epsilon}^{-}).
\]
Decompose:
\[
\tilde{\mathbf{X}}_\ell = (\mathbf{I} - \mathbf{P}_\ell) \mathbf{X}_\ell^{\text{real}} + (\mathbf{I} - \mathbf{P}_\ell) \mathbf{X}_\ell^{\text{hall}} + (\mathbf{I} - \mathbf{P}_\ell) \boldsymbol{\epsilon}^{-}.
\]
Since \(\mathbf{P}_\ell \mathbf{X}_\ell^{\text{real}} \approx \mathbf{X}_\ell^{\text{real}}\), we have \((\mathbf{I} - \mathbf{P}_\ell) \mathbf{X}_\ell^{\text{real}} \approx \mathbf{0}\). For the hallucinated semantics, \(\mathbf{X}_\ell^{\text{hall}} = \mathbf{X}_\ell^{\text{hall}, \parallel} + \mathbf{X}_\ell^{\text{hall}, \perp}\), so:
\[
(\mathbf{I} - \mathbf{P}_\ell) \mathbf{X}_\ell^{\text{hall}} = \mathbf{X}_\ell^{\text{hall}, \perp},
\]
since \((\mathbf{I} - \mathbf{P}_\ell) \mathbf{X}_\ell^{\text{hall}, \parallel} = \mathbf{0}\). Thus:
\[
\tilde{\mathbf{X}}_\ell \approx \mathbf{X}_\ell^{\text{hall}, \perp} + (\mathbf{I} - \mathbf{P}_\ell) \boldsymbol{\epsilon}^{-}.
\]
The error is:
\begin{align*}
\mathbb{E} \|\tilde{\mathbf{X}}_\ell - \mathbf{X}_\ell^{\text{hall}, \perp}\|_F^2 
&= \mathbb{E} \| (\mathbf{I} - \mathbf{P}_\ell) \boldsymbol{\epsilon}^{-} \|_F^2 \\
&= \sigma_{-}^2 (D - C) N.
\end{align*}
since \(\boldsymbol{\epsilon}^{-} \sim \mathcal{N}(0, \sigma_{-}^2 \mathbf{I})\) and \(\mathbf{I} - \mathbf{P}_\ell\) projects onto a \((D - C)\)-dimensional subspace.

The difference-based residual is:
\begin{align*}
\tilde{\mathbf{X}}_\ell^{\text{diff}} 
&= \mathbf{X}_\ell^- - \mathbf{X}_\ell^+ \\
&= (\mathbf{X}_\ell^{\text{real}} + \mathbf{X}_\ell^{\text{hall}} + \boldsymbol{\epsilon}^{-}) 
   - (\mathbf{X}_\ell^{\text{real}} + \boldsymbol{\epsilon}^{+}) \\
&= \mathbf{X}_\ell^{\text{hall}} + \boldsymbol{\epsilon}^{-} - \boldsymbol{\epsilon}^{+}.
\end{align*}
The error is:
\[
\tilde{\mathbf{X}}_\ell^{\text{diff}} - \mathbf{X}_\ell^{\text{hall}, \perp} = \mathbf{X}_\ell^{\text{hall}, \parallel} + \boldsymbol{\epsilon}^{-} - \boldsymbol{\epsilon}^{+}.
\]
The expected squared error is:
\[
\mathbb{E} \|\tilde{\mathbf{X}}_\ell^{\text{diff}} - \mathbf{X}_\ell^{\text{hall}, \perp}\|_F^2 = \|\mathbf{X}_\ell^{\text{hall}, \parallel}\|_F^2 + \mathbb{E} \| \boldsymbol{\epsilon}^{-} \|_F^2 + \mathbb{E} \| \boldsymbol{\epsilon}^{+} \|_F^2,
\]
since cross terms vanish under independence. Given \(\mathbb{E} \| \boldsymbol{\epsilon}^{-} \|_F^2 = \sigma_{-}^2 D N\) and \(\mathbb{E} \| \boldsymbol{\epsilon}^{+} \|_F^2 = \sigma_{+}^2 D N\):
\[
\mathbb{E} \|\tilde{\mathbf{X}}_\ell^{\text{diff}} - \mathbf{X}_\ell^{\text{hall}, \perp}\|_F^2 = \|\mathbf{X}_\ell^{\text{hall}, \parallel}\|_F^2 + \sigma_{-}^2 D N + \sigma_{+}^2 D N.
\]

Comparing the errors, the hallucinatory component has error \(\sigma_{-}^2 (D - C) N\), while the difference-based residual has error \(\|\mathbf{X}_\ell^{\text{hall}, \parallel}\|_F^2 + \sigma_{-}^2 D N + \sigma_{+}^2 D N\). Since \(D - C < D\), we have \(\sigma_{-}^2 (D - C) N < \sigma_{-}^2 D N\). The additional terms \(\|\mathbf{X}_\ell^{\text{hall}, \parallel}\|_F^2 \geq 0\) and \(\sigma_{+}^2 D N \geq 0\) ensure:
\[
\sigma_{-}^2 (D - C) N < \|\mathbf{X}_\ell^{\text{hall}, \parallel}\|_F^2 + \sigma_{-}^2 D N + \sigma_{+}^2 D N.
\]
The $\tilde{\mathbf{X}}_\ell$ isolates \(\mathbf{X}_\ell^{\text{hall}, \perp}\), corresponding to unique hallucinated concepts, and its error is unaffected by \(\mathbf{X}_\ell^{\text{hall}, \parallel}\), ensuring robustness.

\section{Implementation Details}\label{sec:implementation}
We follow the experimental setup in previous work~\cite{yang2025nullu}. During the decoding stage, beam search is employed with \textit{num-beams} set to 3, indicating that three candidate sequences are retained at each step. For all models, the editing layers are selected within the range $\ell \in \text{range}(16,32)$. The number of retained principal components is set to top-$C = 2500$, and the number of weights with the highest similarity scores is set to top-$K = 6000$. All experiments are conducted on a single A40 GPU (40G).

We construct faithful and hallucination-inducing contrastive pairs following the protocol in~\cite{yang2025nullu, zhou2023analyzing}.
For each image and prompt, an auxiliary LLM is used to generate a faithful response grounded in the image and a hallucination-inducing response that introduces non-existent objects or attributes, and this construction is performed once offline before editing.

\begin{table*}[htbp]
\centering
\caption{The OPOPE evaluation results on MSCOCO dataset of LVLMs with different methods for mitigating OH. Higher accuracy, precision, and F1 score indicate better performance. Bold indicates the best result of all methods.}
\tabstyle{1.0pt}
\begin{tabular}{ l | c c c | c c c | c c c}
\toprule
\multirow{2}{*}{\textbf{Method}} 
&\multicolumn{3}{c|}{MiniGPT-4}
&\multicolumn{3}{c|}{mPLUG-Owl2}
&\multicolumn{3}{c}{LLaVA-1.5} \\
\cmidrule{2-10}
&{Accuracy} &{Precision}  &{F1 score} 
&{Accuracy} &{Precision}  &{F1 score} 
&{Accuracy} &{Precision}  &{F1 score}      \\ 
\midrule
Greedy 
& $\text{66.78}_{\pm\text{1.27}}$ & $\text{90.43}_{\pm\text{25.1}}$ & $\text{85.79}_{\pm\text{18.7}}$
& $\text{69.77}_{\pm\text{1.18}}$ & $\text{91.07}_{\pm\text{17.8}}$ & $\text{87.45}_{\pm\text{13.9}}$
& $\text{70.56}_{\pm\text{1.51}}$ & $\text{91.08}_{\pm\text{20.6}}$ & $\text{87.72}_{\pm\text{16.3}}$ \\
DoLA 
& $\text{67.06}_{\pm\text{1.19}}$ & $\text{90.84}_{\pm\text{23.1}}$ & $\text{86.22}_{\pm\text{17.3}}$
& $\text{70.17}_{\pm\text{1.69}}$ & $\text{91.97}_{\pm\text{24.5}}$ & $\text{88.30}_{\pm\text{19.2}}$
& $\text{70.69}_{\pm\text{1.50}}$ & $\text{90.87}_{\pm\text{19.8}}$ & $\text{87.59}_{\pm\text{15.74}}$ \\
OPERA 
& $\text{67.26}_{\pm\text{1.04}}$ & $\text{90.76}_{\pm\text{20.0}}$ & $\text{86.25}_{\pm\text{15.0}}$
& $\text{69.26}_{\pm\text{0.45}}$ & $\text{93.06}_{\pm\text{8.01}}$ & $\text{88.83}_{\pm\text{6.14}}$
& $\text{69.73}_{\pm\text{1.34}}$ & $\text{91.10}_{\pm\text{19.4}}$ & $\text{87.46}_{\pm\text{15.3}}$ \\
VCD 
& $\text{65.78}_{\pm\text{0.96}}$ & $\text{90.02}_{\pm\text{20.7}}$ & $\text{85.00}_{\pm\text{15.1}}$
& $\text{69.81}_{\pm\text{0.65}}$ & $\text{92.70}_{\pm\text{11.0}}$ & $\text{88.76}_{\pm\text{8.49}}$
& $\text{70.67}_{\pm\text{1.22}}$ & $\text{91.62}_{\pm\text{16.7}}$ & $\text{88.19}_{\pm\text{13.3}}$ \\
LURE 
& $\text{68.14}_{\pm\text{0.99}}$ & $\text{90.95}_{\pm\text{17.34}}$ & $\text{86.76}_{\pm\text{13.2}}$
& $\text{69.24}_{\pm\text{1.60}}$ & $\text{90.54}_{\pm\text{23.3}}$ & $\text{86.85}_{\pm\text{18.2}}$
& $\text{70.00}_{\pm\text{1.53}}$ & $\text{90.89}_{\pm\text{21.9}}$ & $\text{87.38}_{\pm\text{17.3}}$ \\
HALC 
& $\text{66.76}_{\pm\text{0.68}}$ & $\text{91.95}_{\pm\text{15.0}}$ & $\text{86.92}_{\pm\text{11.1}}$
& $\text{70.12}_{\pm\text{0.98}}$ & $\text{91.94}_{\pm\text{15.1}}$ & $\text{88.26}_{\pm\text{11.8}}$
& $\text{70.59}_{\pm\text{0.82}}$ & $\text{92.94}_{\pm\text{12.1}}$ & $\text{89.22}_{\pm\text{9.55}}$ \\
Nullu
&$\text{\text{68.81}}_{\pm \text{0.59}}$ &$\text{\text{96.49}}_{\pm \text{0.85}}$ &$\text{\text{91.21}}_{\pm \text{0.65}}$  
&$\text{\text{71.04}}_{\pm \text{1.07}}$ &$\text{\text{96.30}}_{\pm \text{0.39}}$ &$\text{\text{92.05}}_{\pm \text{0.72}}$  
&$\text{\text{72.52}}_{\pm \text{0.14}}$ &$\text{\text{94.46}}_{\pm \text{0.08}}$ &$\text{\text{92.79}}_{\pm \text{0.14}}$ 
\\ \midrule
\rowcolor{mygray}
\textbf{\ours}
&$\text{\textbf{69.79}}_{\pm \text{0.60}}$ &$\text{\textbf{97.83}}_{\pm \text{0.16}}$ &$\text{\textbf{92.60}}_{\pm \text{0.08}}$ 
&$\text{\textbf{72.86}}_{\pm \text{1.00}}$ &$\text{\textbf{97.78}}_{\pm \text{0.03}}$ &$\text{\textbf{93.50}}_{\pm \text{0.02}}$ 
&$\text{\textbf{73.50}}_{\pm \text{0.45}}$ &$\text{\textbf{96.32}}_{\pm \text{0.05}}$ &$\text{\textbf{93.40}}_{\pm \text{0.07}}$ 
\\ 
\bottomrule
\end{tabular}
\label{tab:opope_results}
\end{table*}

\begin{table*}[!ht]
\centering
\caption{Comparison of hallucination mitigation methods in terms of object coverage (\textbf{Cover}) and hallucination rate (\textbf{HalRate}). Higher \textbf{Cover} and lower \textbf{HalRate} indicate better performance. Bold numbers denote the best results per model.
}
\tabstyle{5pt}
\begin{tabular}{ l | c c | c c | c c }
\toprule
\multirow{2}{*}{\textbf{Method}} 
&\multicolumn{2}{c|}{MiniGPT-4}
&\multicolumn{2}{c|}{mPLUG-Owl2}
&\multicolumn{2}{c}{LLaVA-1.5} \\
\cmidrule{2-7}
&{Cover}~($\uparrow$) &{HalRate}~($\downarrow$)
&{Cover}~($\uparrow$) &{HalRate}~($\downarrow$)
&{Cover}~($\uparrow$) &{HalRate}~($\downarrow$) \\ 
\midrule
Greedy & 65.2 & 69.2 & 53.5 & 41.8 & 50.5 & 26.4 \\
VCD    & 52.1 & 46.1 & 63.2 & 70.9 & 50.9 & 25.3 \\
HALC   & 62.6 & 65.9  & 52.7 & 39.2 & 49.9 & 25.9 \\
Nullu  & 57.7 & 70.1  & 53.2 & 41.6 & 49.0 & 26.3 \\
\midrule
\rowcolor{mygray}
\textbf{\ours} & $\textbf{63.2}$ & $\textbf{45.7}$ & $\textbf{53.4}$ & $\textbf{38.5}$ & $\textbf{51.2}$ & $\textbf{25.1}$ \\
\bottomrule
\end{tabular}
\label{tab:AMBER}
\end{table*}

\begin{table*}[!ht]
\centering
\caption{Evaluation of different methods on hallucination severity and grounding quality. Higher \textbf{Score} and lower \textbf{HalRate} reflect stronger grounding. Bold values highlight the best performance for each model.
}
\tabstyle{5pt}
\begin{tabular}{ l | c c | c c | c c }
\toprule
\multirow{2}{*}{\textbf{Method}} 
&\multicolumn{2}{c|}{MiniGPT-4}
&\multicolumn{2}{c|}{mPLUG-Owl2}
&\multicolumn{2}{c}{LLaVA-1.5} \\
\cmidrule{2-7}
&\textbf{Score}~($\uparrow$) &\textbf{HalRate}~($\downarrow$)
&\textbf{Score}~($\uparrow$) &\textbf{HalRate}~($\downarrow$)
&\textbf{Score}~($\uparrow$) &\textbf{HalRate}~($\downarrow$) \\ 
\midrule
Greedy & 1.61 & 0.71  & 2.00 & 0.65 & 2.26 & 0.60 \\
VCD    & 1.68 & 0.71  & 2.15 & 0.62 & 1.97 & 0.67 \\
HALC   & 1.66 & 0.71  & 2.04 & 0.64 & 2.27 & 0.60 \\
Nullu  & 1.32 & 0.77  & 2.19 & 0.61 & 1.61 & 0.71 \\
\midrule
\rowcolor{mygray}
\textbf{\ours} & $\textbf{1.69}$ & $\textbf{0.79}$ & $\textbf{2.23}$ & $\textbf{0.59}$ & $\textbf{1.57}$ & $\textbf{0.58}$ \\
\bottomrule
\end{tabular}
\label{tab:MMHal}
\end{table*}

\begin{figure}[!ht]
    \centering
    \includegraphics[width=0.95\linewidth]{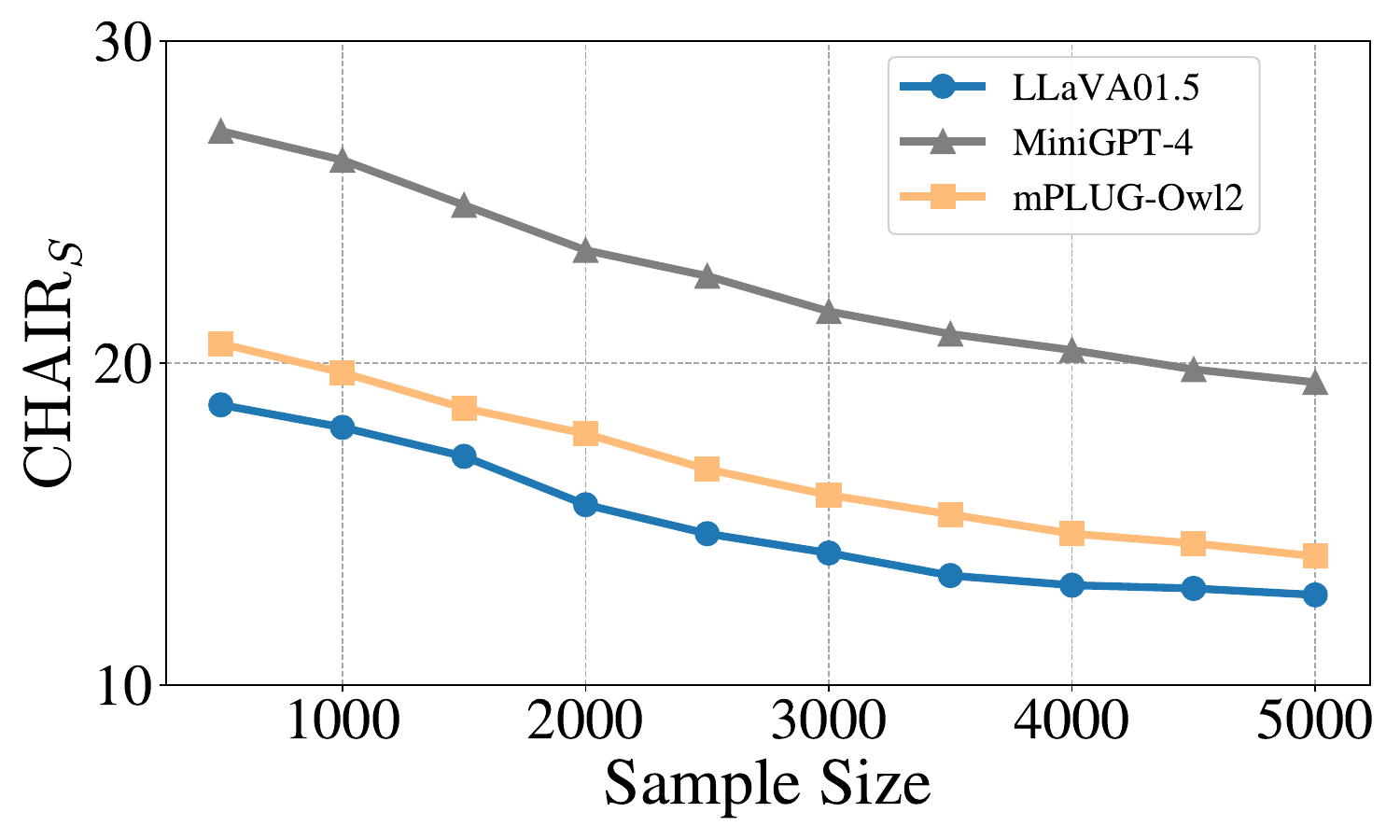}
    \caption{Comparisons of $\text{CHAIR}_{S}$ trends under varying the number of samples.}
    \label{fig:editnumber}
\end{figure}

\begin{table}[t]
    \centering
    \tabstyle{1pt}
    \caption{Analysis of selected weights in model.}
    \vspace{-5pt}
    \label{tab:selected_weights}
    \begin{tabular}{ccccc}
        \toprule
        Values of $K$ & POPE Acc. & CHAIR$_S$ & CHAIR$_I$ & MME \\
        \midrule
        $1 \times 10^3$ & 83.47 & 14.5 & 5.6 & 613.23 \\
        $2 \times 10^3$ & 83.73 & 14.3 & 5.3 & 616.64 \\
        $3 \times 10^3$ & 84.05 & 14.0 & 5.1 & 619.92 \\
        $4 \times 10^3$ & 84.37 & 13.6 & 4.8 & 622.20 \\
        $5 \times 10^3$ & 84.60 & 13.1 & 4.4 & 628.68 \\
        $6 \times 10^3$ & 84.90 & 13.3 & 4.6 & 623.89 \\
        $7 \times 10^3$ & 84.85 & 13.2 & 4.6 & 624.02 \\
        $8 \times 10^3$ & 84.84 & 13.3 & 4.5 & 623.78 \\
        \bottomrule
    \end{tabular}
\end{table}

\section{Additional Results} \label{sec:opope} 

\subsection{Performance Evaluation}\label{sec:more_res}

\noindent{\textbf{Results on OPOPE.}} To further validate our method, we evaluate on the OPOPE benchmark, which performs offline verification by checking whether generated captions mention pre-sampled positive or negative objects. Compared to POPE~\cite{li2023evaluating}, OPOPE is more challenging as models must first describe the full image before grounding specific objects in text. Following standard protocol, we report Accuracy, Precision, and F$_\beta$ ($\beta{=}0.2$)~\cite{chen2024halc} averaged across \textit{random}, \textit{popular}, and \textit{adversarial} settings. Results are summarized in Table~\ref{tab:opope_results}.

\begin{itemize}[leftmargin=*]
    \item \textbf{Obs: MPD consistently achieves the best or comparable performance across all models.}  
On LLaVA-1.5, MPD achieves the highest accuracy (73.50) and F1 score (93.40), outperforming Nullu by +0.98 and +0.61, respectively. On MiniGPT-4, we reach the best F1 score (92.60) and competitive accuracy (69.79). On mPLUG-Owl2, our method achieves the highest F1 score (93.50), surpassing Nullu and HALC. These results demonstrate that our method remains effective across models and metrics, improving object grounding accuracy while maintaining precision.

\end{itemize}



\noindent\textbf{Results on AMBER.}  
AMBER~\cite{wang2023amber} is a generative evaluation benchmark that provides fine-grained object-level annotations across 1,000 images. It evaluates hallucination from multiple perspectives, including object coverage (Cover), hallucination rate (HalRate), and alignment with cognition. Table~\ref{tab:AMBER} reports object-level evaluation results on three LVLMs using AMBER metrics.

\begin{itemize}[leftmargin=*]
    \item \textbf{Obs: MPD achieves the best trade-off between coverage and hallucination suppression.}  
    Compared with baselines, our method achieves the highest object coverage across all models (63.2 for MiniGPT-4, 53.4 for mPLUG-Owl2, and 51.2 for LLaVA-1.5), while also obtaining the lowest hallucination rates (45.7, 38.5, and 25.1, respectively). This suggests improved grounding performance without sacrificing content completeness.
\end{itemize}

\noindent\textbf{Results on MMHalBench.}  
MMHalBench~\cite{sun2023aligning} is a GPT-4-assisted question-answering benchmark that spans 12 object topics. Responses are rated by GPT-4 on a scale from 0 to 6, and those scoring below 3 are considered hallucinated. The metrics include average score and hallucination rate. Table~\ref{tab:MMHal} presents GPT-4 evaluation results following the MMHalBench protocol.

\begin{itemize}[leftmargin=*]
    \item \textbf{Obs: MPD improves response quality while reducing hallucinations.}  
    Our method achieves the highest GPT-4 scores across all models (1.69 for MiniGPT-4, 2.23 for mPLUG-Owl2, and 1.57 for LLaVA-1.5), while maintaining the lowest hallucination rates (0.79, 0.59, and 0.58, respectively). These results indicate that our editing strategy produces more faithful and informative responses without degrading linguistic quality.
\end{itemize}

\noindent{\textbf{Results on V* Bench, MMMU and	MathVersion}}
 To examine generalization beyond object-centric hallucinations, we further evaluate MPD on V* Bench~\cite{VstarBench}, which focuses on attribute and spatial consistency, as well as MMMU~\cite{MMMU} and MathVision~\cite{MathVision}, which require multimodal reasoning over charts and abstract visual concepts.
As shown in Table~\ref{tab:non-object}, \ours improves the Attribute, Spatial, and Overall scores on V* Bench over both the base model and Nullu, indicating stronger fine-grained visual consistency.
On MMMU and MathVision, although absolute performance remains modest due to task difficulty, MPD still yields consistent gains over the baselines, suggesting that our editing strategy does not harm (and can slightly benefit) broader multimodal reasoning ability.
\begin{table}[!t]
  \centering
  \tabstyle{0.5pt}
  \caption{Evaluation on non-object hallucination benchmarks.}
  \begin{tabular}{l|ccc|c|c}
    \toprule
    \multirow{2}{*}{Model} 
      & \multicolumn{3}{c|}{V* Bench} 
      & \multirow{2}{*}{MMMU} 
      & \multirow{2}{*}{MathVersion} \\ \cmidrule{2-4}
    
      & Attribute & Spatial & Overall &  &  \\
    \midrule
    LLaVA-1.5-7B     & 43.47 & 56.57 & 48.68 & 35.7 & 10.1 \\
    Nullu         & 45.08 & 57.26 & 49.15 & 35.9 & 10.3 \\
    \textbf{\ours}& 47.73 & 58.91 & 51.03 & 36.5 & 11.0 \\
    \bottomrule
  \end{tabular}\label{tab:non-object}
\end{table}


\noindent\textbf{Effect of Auxiliary LLM Choice.}
Our approach relies on an auxiliary LLM to construct faithful and hallucination contrastive pairs, motivating an examination of the LLM choice.
Following prior representation intervention work, we use GPT-3.5 in our main experiments to ensure comparability.
We then replace GPT-3.5 with GPT-5.1 and reconstruct all contrastive pairs under the same protocol.
Table~\ref{tab:gpt} shows that MPD yields consistent reductions in CHAIR$_S$ and CHAIR$_I$ and comparable BLEU scores across different auxiliary LLMs.
\begin{table}[!t]
  \centering
  \tabstyle{5pt}
  \caption{Effect of auxiliary LLM choice for contrastive pair construction.}
  \begin{tabular}{lccc}
    \toprule
    Model 
    & CHAIR$_S \downarrow$ 
    & CHAIR$_I \downarrow$ 
    & BLEU $\uparrow$ \\
    \midrule
    LLaVA-1.5-7B        & 20.40 & 7.08 & 15.72 \\
    \textbf{\ours} (GPT-3.5)   & 12.80 & 4.21 & 15.31  \\
    \textbf{\ours} (GPT-5.1)   & 12.60 & 4.21 & 15.02  \\
    \bottomrule
  \end{tabular}\label{tab:gpt}
\end{table}

\subsection{Impact of sample size.}\label{sec:abla_size}
We vary the sample size from 500 to 5000 and evaluate CHAIR$_S$ across three LVLMs. The results are shown in Figure~\ref{fig:editnumber}.

\begin{itemize}[leftmargin=*]
    \item \textbf{Obs: More editing samples improve performance, while small-scale edits remain effective.}  
    CHAIR$_S$ scores consistently decrease as the number of editing samples increases, indicating enhanced hallucination suppression. Notably, even with only 500 samples, our method yields clear gains over the original models, demonstrating strong data efficiency.
\end{itemize}

\subsection{Effect of selected weights.}\label{sec:abla_weights}
To understand how edited weight count affects hallucination suppression, we ablate top-$K$ selection, where $K$ is the number of weight rows aligned with hallucination semantics. Table~\ref{tab:selected_weights} reports results on LLaVA-1.5-7B across POPE, CHAIR, and MME benchmarks. 

\begin{itemize}[leftmargin=*]
    \item \textbf{Obs 10: Editing more weights improves performance up to a saturation point.}  
    As $K$ increases from $1 \times 10^3$ to $5 \times 10^3$, POPE accuracy steadily improves from 83.47 to 84.60, CHAIR$_{S}$/CHAIR$_{I}$ drop from 14.5/5.5 to 13.1/4.4, and MME score rises from 613.23 to 628.68. This consistently indicates that increasing the number of updated weights enhances hallucination suppression. However, further increasing $K$ beyond $5 \times 10^3$ yields only marginal gains, suggesting that the generation of hallucinations are highly concentrated in a limited subset of parameters, and overly broad edits may lead to diminishing returns.
\end{itemize}

\section{Effiency Comparison.}
\begin{table}[htbp]
    \begin{center}
    \tabstyle{4pt}
    \caption{Efficiency and hallucination comparison across methods in terms of latency, GPU memory, and CHAIR$_S$. Experiments are evaluated on a single RTX A40 GPU.}
    \label{tab:efficiency}
    \begin{tabular}{lcccc}
    \toprule
    Method  & Avg. Latency $\downarrow$ & GPU Memory $\downarrow$ & CHAIR$_S$ $\downarrow$  \\ 
    \midrule
    Greedy  &  3.1 s &  14758 MB   & 20.4\\
    VCD &  6.5 s &  16538 MB    & 20.3 \\
    OPERA & 23.5 s  &  23742 MB  & 17.5 \\
    HALC & 20.1 s  &  22135 MB & 16.9 \\
    Nullu & 3.9 s  &  15183 MB & 15.2 \\
    \rowcolor{gray!20}
    \textbf{\ours} &  3.7 s   & 15019 MB  & 12.8\\
    \bottomrule
    \end{tabular}
\end{center}
\end{table}
Table~\ref{tab:efficiency} presents a comparison of inference latency, peak GPU memory usage, and CHAIR$_S$ scores on the CHAIR benchmark using LLaVA-1.5 with a maximum token length of 128.

\begin{itemize}[leftmargin=*]
    \item \textbf{Obs: MPD achieves the best hallucination suppression with negligible runtime overhead.}  
    While our method involves a lightweight parameter editing step prior to inference, it introduces only minimal latency increase compared to the greedy baseline (3.7s vs. 3.1s) and requires no additional modules during decoding. Moreover, it outperforms all baselines in CHAIR$_S$ (48.8), including optimization-heavy approaches such as HALC and OPERA, demonstrating an effective trade-off between performance and efficiency.
\end{itemize}

\section{Case Studies}
 Figure~\ref{fig:llava_case} shows case studies from LLaVA-Bench under identical prompts. The original model often hallucinates objects inconsistent with the image, such as \textit{\textcolor{red}{Mount Fuji}}'' for Diamond Head, or \textit{\textcolor{red}{fox}}'' and \textit{\textcolor{red}{musical instruments}}'' in a scene with no such elements. It also introduces non-existent food items like \textit{\textcolor{red}{broccoli}}'' and ``\textit{\textcolor{red}{toast}}''. In contrast, our method generates responses that align with the visual content, avoiding co-occurrence-driven hallucinations. 
\begin{figure*}[!ht]
\includegraphics[width=0.98\linewidth]{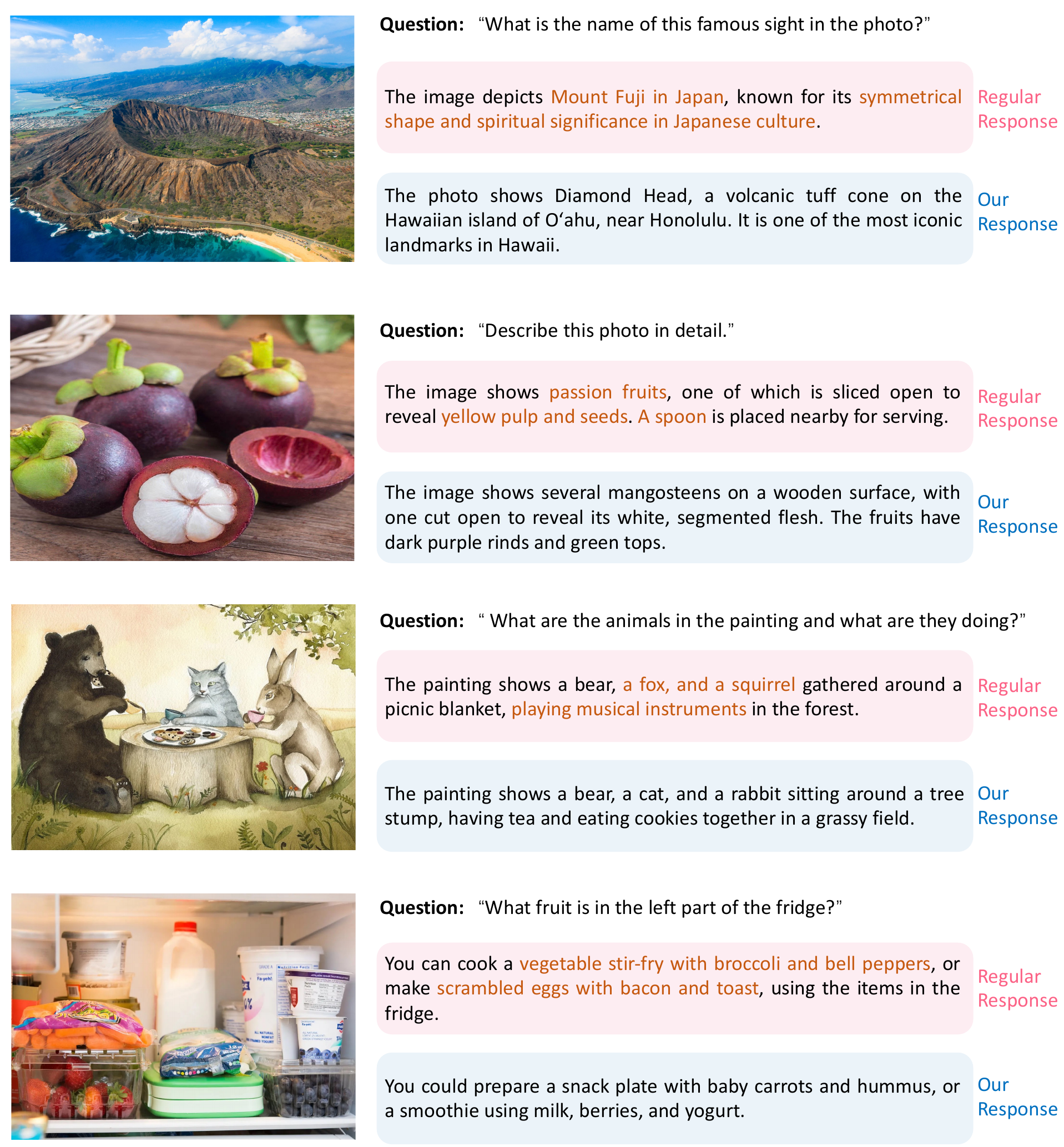} 
\caption{{Case studies on the LLaVA-Bench benchmark}. We compare the responses generated by regular decoding and our method using LLaVA-1.5.}
\label{fig:llava_case}
\end{figure*}

\end{document}